\documentclass[11pt]{article}

\usepackage[margin=1in]{geometry}
\usepackage{amsmath,amssymb,amsthm}
\usepackage{booktabs}
\usepackage{enumitem}
\usepackage{tabularx}
\usepackage{algorithm}
\usepackage{algpseudocode}
\usepackage{multirow}
\usepackage{tikz}
\usepackage[round]{natbib}
\usepackage[colorlinks=true,linkcolor=blue,citecolor=blue,urlcolor=blue]{hyperref}

\usepackage{amsmath, amsthm}
\newtheorem{proposition}{Proposition}

\newtheorem{definition}{Definition}

\newcommand{\E}{\mathbb{E}}
\newcommand{\R}{\mathcal{R}}
\newcommand{\Train}{\mathcal{T}}
\newcommand{\Val}{\mathcal{V}}
\newcommand{\Test}{\mathcal{H}}
\newcommand{\elloss}{\ell}
\newcommand{\Rhat}{\widehat{\mathcal{R}}}
\newcommand{\pverr}{\mathrm{pverr}}

\title{An Explicit Feedback Contract for Iterative ML-Agent Harnesses}
\author{
\parbox{\dimexpr\textwidth-2\tabcolsep\relax}{\centering\normalsize
Kehui Yao\textsuperscript{1}, Chuanbo Hu\textsuperscript{1}, Jiao Chen\textsuperscript{1},\\
Ramin Giahi\textsuperscript{1}, Reza Yousefi Maragheh\textsuperscript{1}, Lalitesh Morishetti\textsuperscript{1},\\
Ziwei Hong\textsuperscript{2}, Yiyun Su\textsuperscript{3}, Jianpeng Xu\textsuperscript{1},\\
Evren Korpeoglu\textsuperscript{1}, Kannan Achan\textsuperscript{1}\\[0.8em]
\small\textsuperscript{1}Walmart Global Tech\\
\textsuperscript{2}Lehigh University\\
\textsuperscript{3}Rutgers University
}}
\date{\today}

\begin{document}
\maketitle

\begin{abstract}
Iterative ML agents develop models by editing code, running experiments, reading a validation score, and choosing what to try next. In current harnesses this loop is served by an exact validation score returned after every attempt, which leaves the validation split open to adaptive selection.
In this work, we introduce an explicit feedback contract for ML-agent harnesses, which specifies what an agent may observe, when iteration should stop, and when the test split may be touched. The contract releases a validation score only when a submission clears a noise threshold computed from the paired standard error, so no margin has to be tuned.
We ran experiments on click-through-rate prediction, learning to rank, and a control task with no signal to find. The control is the sharpest case, because the true gain there is zero by construction: an agent reading exact scores ended by reporting a gain of $0.0157$, while under the contract it reported $0.0033$. On both real tasks the contract narrows the validation--test gap, and test performance is slightly higher rather than lower. The largest difference is in when the agents stop. Exact feedback used its entire budget of thirty submissions on two of the three tasks, including the one where nothing can be learned, while the contract halted after four to twelve. A run that exhausts its budget is usually read as diligence. It is better read as a symptom: when the feedback channel is unbounded, noise keeps supplying improvements to chase.

\end{abstract}

\section{Introduction}

Language-model agents are no longer only writing code. They now run experiments: editing code, launching training runs, reading a validation score, and choosing the next edit. A growing body of systems and benchmarks evaluate this loop \citep{romeraparedes2024funsearch,jiang2025aide,lu2024aiscientist,wang2026self, chan2024mlebench,huang2024mlagentbench, pinchuk2026tmlbench}. They differ in what the agent proposes but not in what it observes: an exact validation score after every attempt.

The feedback interface matters more than it appears. Each validation score is an ordinary held-out estimate, but the sequence produced by iteratively keeping score-improving edits is not \citep{dwork2015preserving}. Validation scores are computed on a finite sample and carry noise. From inside the loop, a real improvement and a fluctuation of the same size look identical, and the agent keeps the edit in both cases. After enough rounds, the score it reports describes one particular split rather than the population. Figure~\ref{fig:null-gap} shows the effect on data where the truth is known.

\begin{figure}[t]
  \centering
  \includegraphics[width=0.8\linewidth]{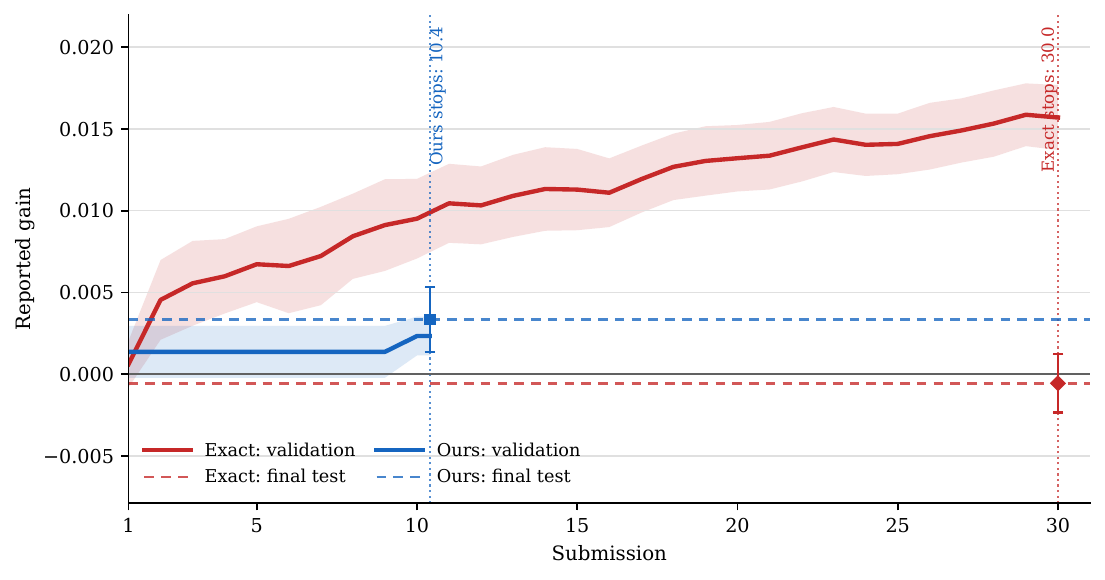}
  \caption{An agent iterating against a validation split whose labels have been permuted: nothing can beat the base rate, so the true gain is zero and a flat line is correct. Plotted is the gain the agent is shown after each submission, over five seeds with standard-error bands. An exact score after every attempt, as current harnesses give, produces the red curve; the contract produces the blue one. Dashed lines mark the one-shot final test gain, never visible during iteration, and dotted lines mark the mean stopping round. The blue trajectory and its final-test marker end at the mean stopping round, $10.4 \pm 4.4$, whereas Exact continues to the $30.0 \pm 0.0$ submission cap.}
  \label{fig:null-gap}
\end{figure}

This is not a hypothetical worry. Held-out accuracy inflates measurably when a benchmark is reused across a research community \citep{recht2019imagenet}. A meta-analysis of Kaggle competitions finds adaptive overfitting where submission counts are large relative to holdout size \citep{roelofs2019meta}. Agents change that ratio. A competition may collect a few thousand submissions from hundreds of teams over several months. A single agent can issue as many against one validation split in an afternoon. There is no rate limit, and no one is watching for the moment the gains stop being real. 

The statistical problem is well understood. Adaptive data analysis shows that a holdout can be reused when the released answers are stable, and differential privacy is one way to achieve this \citep{dwork2015preserving, dwork2015generalization, dwork2015reusable}. There is also a matching hardness result: no efficient estimator can answer every adaptively chosen query accurately \citep{hardt2014preventing}. Practical mechanisms therefore aim at something weaker. For leaderboards, \citet{blum2015ladder} aim at the best score so far, and reach it with a simple rule. Release a new score only when a submission beats the previous best by a margin. Otherwise repeat the old score.
 
An agent is different from a competition submitter in three ways that a threshold alone does not address. First, an agent can retrain and inspect intermediate files, so a contract must say which files are readable. Second, there is no closing date, so the protocol needs a stopping rule. Third, the agent works inside the repository, so anything the harness does not actively withhold—including validation predictions—can leak. These three gaps turn a single-mechanism question into a protocol-design problem.

We propose an explicit feedback contract for iterative ML-agent harnesses, and make three contributions. (1) Agent iteration is identified as a case of adaptive overfitting, and given a precise form. (2) The contract specifies what an agent may see, when it should stop, and when the test set may be used. It needs no parameter tuning, and its stopping rule is guaranteed to fire: the agent halts because no further improvement can be certified, not because a budget ran out. (3) On two real tasks and one control task, an ordinary coding agent overfits the validation data without trying to. The contract narrows the gap at no cost to final performance, and stops after a fraction of the submissions an unbounded interface uses.

\section{Related Work}

\subsection{Agentic ML systems and their evaluation interfaces}
An ML-agent harness settles three things before a run: data splits, budget, and what the agent is told after each attempt. Benchmarks settle the first two with care. Splits are fixed, budgets are capped, and a private portion is held back until the end \citep{chan2024mlebench, huang2024mlagentbench, pinchuk2026tmlbench}. The third is left implicit. An agent receives the metric it asked for, at full precision, as often as it asks. That was inherited rather than chosen. In single-shot evaluation, reporting a validation score exactly is the right thing to do. The systems that placed an agent inside the loop \citep{romeraparedes2024funsearch, jiang2025aide, lu2024aiscientist, wang2026self, zhang2026agenticrs} kept the interface and replaced the submitter with one that queries it hundreds of times.

\subsection{Mechanisms for adaptive reuse}

Two routes exist for reusing a holdout under adaptive queries: perturbing the released answers with a privacy budget \citep{dwork2015preserving, dwork2015reusable}, and thresholding the release so that only scores that clear a margin are reported \citet{blum2015ladder}. Both assume that the submitter cannot access the raw data—an assumption that breaks when the submitter is an agent with file-system access. Table~\ref{tab:related-work} shows the gap on the agent side. What it cannot show is why agentic ML systems never restricted their feedback channel, even though mechanisms for doing so already existed in adaptive data analysis.

\begin{table}[t]
\centering
\small
\setlength{\tabcolsep}{5pt}
\begin{tabular}{@{}lccc@{}}
\toprule
\textbf{Line of work} & \textbf{Iterative} & \textbf{Bounded}
& \textbf{Held-out} \\
& \textbf{agent} & \textbf{feedback} & \textbf{reporting} \\
\midrule
\multicolumn{4}{@{}l}{\emph{Agentic ML systems}} \\
\quad FunSearch~\cite{romeraparedes2024funsearch}
  & \checkmark & --- & --- \\
\quad AIDE~\cite{jiang2025aide}
  & \checkmark & --- & --- \\
\quad AI Scientist~\cite{lu2024aiscientist}
  & \checkmark & --- & --- \\
\quad Agentic recommenders~\cite{wang2026self,zhang2026agenticrs}
  & \checkmark & --- & --- \\[2pt]
\multicolumn{4}{@{}l}{\emph{ML-agent benchmarks}} \\
\quad MLE-bench~\cite{chan2024mlebench}
  & \checkmark & --- & $\circ$ \\
\quad MLAgentBench~\cite{huang2024mlagentbench}
  & \checkmark & --- & $\circ$ \\
\quad TML-Bench~\cite{pinchuk2026tmlbench}
  & \checkmark & --- & $\circ$ \\[2pt]

\midrule
\textbf{Ours}
  & \checkmark & \checkmark & \checkmark \\
\bottomrule
\end{tabular}
\caption{Positioning of this work. \emph{Iterative agent}: an agent revises code and experiments in response to feedback. \emph{Bounded feedback}: the information released per submission is restricted, with a stated guarantee. \emph{Held-out reporting}: reported numbers come from a split never queried during development. $\circ$ marks partial coverage.}
\label{tab:related-work}
\end{table}

\begin{figure}[tbp]
  \centering
  \includegraphics[width=\linewidth]{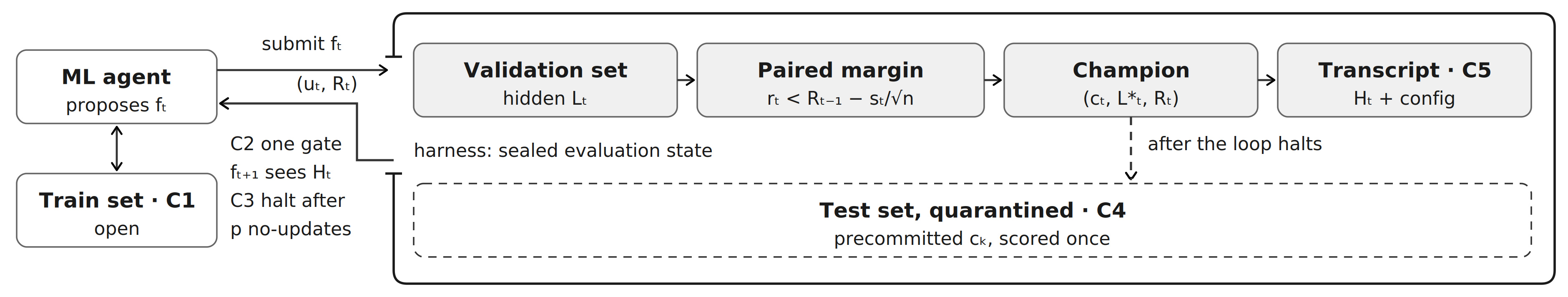}
  \caption{The feedback contract as a sealed harness. Each round it submits one candidate $f_t$ and receives only the pair $(u_t, R_t)$. Inside the harness, top to bottom in execution order: hidden validation losses, threshold rule, champion state, and audit transcript $H_t$. After $p$ consecutive no-updates the loop halts and the final champion $c_k$ is evaluated once on the test split $H$.}
  \label{fig:harness-contract}
\end{figure}

\section{Problem Formulation}
\label{sec:formulation}

\paragraph{Setting.} An agent develops a model over $k$ rounds against a harness. The harness holds three disjoint splits $\Train$, $\Val$, and $\Test$. The agent uses $\Train$ freely. It can reach $\Val$ only through the harness, which keeps the labels. It does not touch $\Test$ until it has stopped.

Each round has three steps. The agent submits one candidate, written $f_t$. The harness scores that candidate on $\Val$ and keeps the score to itself. The harness then releases a single pair $(u_t,R_t)$: whether the released score changed, and what it now stands at. Nothing else comes back. The transcript $H_t$ of pairs released up to round $t$ is everything that leaves the harness. What rule decides $(u_t,R_t)$ is the subject of Section~\ref{sec:contract}.

\paragraph{Adaptivity.} The agent writes $f_t$ after reading $H_{t-1}$.
 
\begin{definition}[Admissible sequence]
\label{def:admissible}
    A sequence $f_1,\ldots,f_k$ is admissible if $f_t$ depends only on $\Train$, the repository state, earlier edits, and $H_{t-1}$, and not on the labels or per-example losses of $\Val$, on validation prediction files, or on any feedback from $\Test$.
\end{definition}

To see why this matters, let $\elloss(f,z)\in[0,1]$ be a bounded loss
on an example $z$, and let
\[
    \R(f)=\E_{z\sim P}\bigl[\elloss(f,z)\bigr]
\]
be the risk of $f$ under the unknown distribution $P$ the data comes
from. This is the quantity of interest, and nobody can observe it.
What the harness computes instead is the average loss over the $n$
examples of $\Val$,
\[
    \Rhat_{\Val}(f)=\frac{1}{n}\sum_{i=1}^{n}\elloss(f,z_i).
\]

For a model fixed in advance, the second estimates the first without bias. Here the model is not fixed in advance. Candidate and sample are dependent, so the guarantee lapses, and the number an agent reports at the end is the smallest score it happened to meet along an adaptively chosen path.

\paragraph{Objective.} A harness cannot promise that every released score is close to the risk of the candidate that produced it, since that is unachievable under adaptivity \citep{hardt2014preventing}. It can promise something weaker and still useful: that the released score stays close to the risk of the best candidate submitted so far.
 
\begin{definition}[Public score error]
\label{def:pverr}
For an admissible sequence with releases $R_1,\ldots,R_k$,
\[
    \pverr(R_1,\ldots,R_k)
    =\max_{t\leq k}\Bigl|\min_{i\leq t}\R(f_i)-R_t\Bigr| .
\]
\end{definition}
 
A release rule should keep $\pverr$ small for every admissible sequence, while leaving $H_k$ informative enough that the agent can still choose among its own candidates. The two requirements pull against each other, and Section~\ref{sec:contract} settles where to stand between them.

\section{The Feedback Contract: Thresholded Best-So-Far Release}
\label{sec:contract}

\subsection{The contract: five clauses}
\label{sec:clauses}

A harness has to decide five things before a run starts; deciding them during one would itself be information. Figure~\ref{fig:harness-contract} shows the arrangement.

\begin{description}
\item[\textnormal{(C1)} Splits.] $\Train$ is open. $\Val$ is reachable only by submitting a candidate; its labels, per-example losses, and prediction files stay inside the harness. $\Test$ is unreachable while the agent iterates.

\item[\textnormal{(C2)} Released feedback.] Each round takes one submission and returns one pair $(u_t,R_t)$. Nothing else leaves. A rejected submission is not told by how much it missed, so the agent cannot steer by the size of its failures.

\item[\textnormal{(C3)} Stopping.] A budget $T_{\max}$ and a patience $p$ are fixed in advance. Iteration halts at the budget or after $p$ consecutive non-updates. Patience counts updates, not score changes: an update has cleared a threshold, whereas a score change need not have.

\item[\textnormal{(C4)} Test quarantine.] A final model is named before any test feedback exists, by default the champion $c_k$. $\Test$ is then evaluated once, and only aggregate metrics are released.

\item[\textnormal{(C5)} Recorded artifacts.] The harness stores the configuration, hidden loss vectors, transcript, seeds, agent identifiers, and final metrics, in enough detail to replay the transcript without rerunning the agent.

\end{description}

\subsection{A parameter-free release rule}
\label{sec:rule}

C2 says a score is released only when a submission is good enough, and an algorithm has to say how good. Fixing a margin $\eta$ in advance does not work: too high discards real gains, too low releases almost every round, and the right value depends on task noise that is unknown beforehand. The rule below computes the margin from the data instead.
 
The harness remembers three things between rounds: the score currently on release, the candidate that earned it, which we call the \emph{champion}, and that candidate's loss on each of the $n$ validation examples. Algorithm~\ref{alg:release} scores a new candidate on the same examples, compares row by row against the champion, and releases a new score only if the improvement exceeds the noise in that comparison.

\begin{algorithm}[t]
\caption{Parameter-free thresholded release}
\label{alg:release}
\begin{algorithmic}[1]
\Require splits $\Train$, $\Val=\{z_i\}_{i=1}^n$, $\Test$; agent
policy $\pi$; budget $T_{\max}$; patience $p$
\State $R_0\gets\infty$, $c_0\gets\emptyset$, $L^\star_0\gets(0,\ldots,0)$, $H_0\gets\emptyset$
\For{$t=1,\ldots,T_{\max}$}
  \State agent proposes $f_t$ from $\Train$, repository state, and $H_{t-1}$
  \State $L_t\gets(\elloss(f_t,z_i))_{i=1}^n$ \Comment{hidden}
  \State $r_t\gets n^{-1}\sum_i L_{t,i}$, \quad $s_t\gets\mathrm{std}(L_t-L^\star_{t-1})$
  \If{$r_t < R_{t-1}-s_t/\sqrt{n}$}
    \State $R_t\gets[r_t]_{1/n}$, \; $c_t\gets f_t$, \; $L^\star_t\gets L_t$, \; $u_t\gets\texttt{update}$
  \Else
    \State $R_t\gets R_{t-1}$, \; $c_t\gets c_{t-1}$, \; $L^\star_t\gets L^\star_{t-1}$, \; $u_t\gets\texttt{no\_update}$
  \EndIf
  \State $H_t\gets H_{t-1}\,\|\,(u_t,R_t)$; release $(u_t,R_t)$
  \State \textbf{break} if the last $p$ events were all \texttt{no\_update}
\EndFor
\State evaluate the precommitted final model once on $\Test$; release aggregate metrics
\end{algorithmic}
\end{algorithm}

The quantity $s_t/\sqrt{n}$ on line 6 is the standard error of the average difference in loss between candidate and champion on the same $n$ examples. A candidate must win by at least one standard error of that difference, so on a noisy task the bar rises by itself and on a clean one it falls. An accepted score is released rounded to a multiple of $1/n$, written $[\cdot]_{1/n}$.

\subsection{Adaptive validity of the released transcript}
\label{sec:validity}

A threshold makes the interaction countable. If the score only moves down by at least $\eta$ and is reported in multiples of $\eta$, it can move at most $1/\eta$ times and take at most $1/\eta$ values, so few transcripts are possible, and a bound that holds for one model fixed
in advance can be applied to each and combined. Proposition~\ref{prop:accuracy} carries that argument out, bounding $\pverr$ (Definition~\ref{def:pverr}): the gap between the released score and the risk of the best candidate submitted so far.

\begin{proposition}[Best-so-far accuracy]
\label{prop:accuracy}
Fix $\eta$ and consider the variant of Algorithm~\ref{alg:release}
that updates when $\Rhat_{\Val}(f_t)<R_{t-1}-\eta$ and releases
$[\Rhat_{\Val}(f_t)]_\eta$. For every admissible sequence
(Definition~\ref{def:admissible}) of length $k$, setting
$\eta=O(n^{-1/3}\log^{1/3}(kn))$ gives
$\pverr=O(n^{-1/3}\log^{1/3}(kn))$ with high probability
\citep{blum2015ladder}.
\end{proposition}

The dependence on $k$ is logarithmic: submitting more often does not move the released score. The same counting bounds how long a run can last, which a competition had no use for but a development loop does.

\begin{proposition}[Termination]
\label{prop:termination}
Under the fixed-margin variant a work batch contains at most
$\lceil 1/\eta\rceil+1$ accepted updates, so a run governed by C3
halts at round
$T_{\mathrm{stop}}\leq\min\{T_{\max},\,p\,(\lceil 1/\eta\rceil+2)\}$.
\end{proposition}

\begin{proof}
The first accepted update sets $R_t$ to a finite value in $[0,1]$, and each subsequent one lowers it by more than $\eta$, so at most $\lceil 1/\eta\rceil$ can follow. Suppose the run has not halted by round $t$. Then no $p$ consecutive rounds are free of updates, so the $U$ update rounds cut $1,\ldots,t$ into at most $U+1$ stretches of non-updates, each shorter than $p$, giving $t<p(U+1)$.
\end{proof}

Stopping is thus part of the protocol rather than a heuristic placed on top of it. Without a margin the same rule never fires: an adaptive submitter reading exact scores can drive the observed loss to $\tfrac12-\Omega(\sqrt{k/n})$ on data with no signal \citep[Thm.~5.1]{blum2015ladder}. Section~\ref{sec:experiments} reports the stopping behaviour of both interfaces.

Both propositions are stated for a fixed margin. Algorithm~\ref{alg:release} uses $\max\{s_t/\sqrt{n},\,\eta_{\min}\}$ with $\eta_{\min}=c\,n^{-1/3}\log^{1/3}(kn)$ fixed before the run from $n$ and $T_{\max}$, so both apply with $\eta=\eta_{\min}$: the floor is what keeps the number of reachable transcripts finite, while above it the margin remains data-dependent and untuned.

\section{Results}
\label{sec:experiments}

\subsection{Experimental Setup}

\paragraph{Datasets.}  Criteo-style click-through-rate prediction is binary classification on a processed dataset with $100{,}000$ training examples, $25{,}000$ validation examples, and an independent $50{,}000$-example test set; we report AUC. MSLR-WEB10K uses the official Fold1 split, with $6{,}000$ training, $2{,}000$ validation, and $2{,}000$ test queries. Query groups remain intact; the base ranker is a \texttt{HistGradient\-Boosting\-Regressor} fitted on $150{,}000$ rows from the training split, and we report macro NDCG@10. Because NDCG decomposes per query rather than per row, the sample size entering Proposition~\ref{prop:accuracy} is the number of validation queries, roughly $2{,}000$, so the guarantee is weaker here than the row count suggests. A third
dataset is a no-signal control, obtained by permuting labels before
splitting and matched in dimension to Criteo: the base rate is population-optimal, so any apparent gain is selection noise by construction.

\paragraph{Conditions.} \emph{Exact} returns $\Rhat_{\Val}(f_t)$ after
every submission, as current harnesses do. \emph{Ours} is
Algorithm~\ref{alg:release}. Patience applies to both, but only
\emph{Ours} can count it on accepted updates; Exact has no update events and falls back on observed improvement. We run three independent attack seeds for Criteo and MSLR and five for the no-signal control.

\paragraph{Metrics.} We report the stopping round, test performance of the precommitted model, and the validation--test gap, defined as validation minus test. A positive gap means the selection signal was optimistic; a negative one means the reverse.

\subsection{Feedback interfaces compared}

\begin{table}[t]
\centering
\small
\setlength{\tabcolsep}{3.5pt}
\begin{tabular}{@{}llccc@{}}
\toprule
\textbf{Task (metric)}
& \textbf{Method}
& \textbf{Stop iter. } $\downarrow$

& \textbf{Test} $\uparrow$
& \textbf{Val.--test gap} $\downarrow$ \\
\midrule
\multirow{2}{*}{Criteo (AUC)}
& Exact
& 30.0 $\pm$ 0.0

& 0.7820 $\pm$ 0.0001
& 0.0060 $\pm$ 0.0001 \\
& Ours
& \textbf{12.3} $\pm$ 4.4

& \textbf{0.7823} $\pm$ 0.0000
& \textbf{0.0055} $\pm$ 0.0000 \\
\midrule
\multirow{2}{*}{MSLR (NDCG@10)}
& Exact
& 27.0 $\pm$ 0.0

& 0.4446 $\pm$ 0.0003
& 0.0134 $\pm$ 0.0003 \\
& Ours
& \textbf{4.3} $\pm$ 0.7

& \textbf{0.4451} $\pm$ 0.0001
& \textbf{0.0123} $\pm$ 0.0001 \\
\midrule
\multirow{2}{*}{No-signal (Gain)}
& Exact
& 30.0 $\pm$ 0.0

& $-0.0006$ $\pm$ 0.0018
& 0.0163 $\pm$ 0.0026 \\
& Ours
& \textbf{10.4} $\pm$ 4.4

& \textbf{0.0033} $\pm$ \textbf{0.0020}
& \textbf{0.0000} $\pm$ \textbf{0.0025} \\
\bottomrule
\end{tabular}
\caption{Results under a maximum of 30 submissions, reported as mean
$\pm$ standard error across independent seeds. Validation and test metrics
are not comparable across tasks. The validation--test gap is validation
minus test, so a larger positive value indicates more optimistic public
leaderboard selection. For the no-signal control, gain is measured relative
to the base-rate predictor, whose population-optimal gain is zero.}
\label{tab:main-results}
\end{table}

Table~\ref{tab:main-results} reports both interfaces on the three
datasets.

\paragraph{The budget stops Exact; the mechanism stops Ours.} Exact runs to the cap on Criteo and the control, and to round 27 on MSLR. Ours halts at a third of the budget or less on all three. This is the behaviour anticipated in Section~\ref{sec:validity}: without a margin observed improvements keep arriving, so a rule watching them has little reason to fire. Stopping is not equally decisive everywhere. The stopping round differs by several submissions across seeds on Criteo and the control, and is tightest on the ranking dataset.
 
\paragraph{Under exact feedback the public score reports progress that does not exist.} On the control the true gain is zero, so any reported gain is an artefact of selection. Exact ends with a reported gain of $0.0157$ that its test score does not support, leaving a gap of $0.0163$. The contract also ends reporting a gain, but one its test score matches, leaving no gap at all. Neither agent was gaming anything: both kept whichever edit appeared to help. The difference lies entirely in what they were told after each attempt.
 
\paragraph{On real data the effect is real but modest.} The
validation--test gap narrows under the contract by roughly a tenth on both real datasets, and test performance is slightly higher rather than lower. The differences are small in absolute terms, and the claim they support is narrower than a better model: the reported score becomes more trustworthy, and nothing is paid for it.

\subsection{Sensitivity to public-validation size}

A larger public set is the obvious alternative remedy, so Table~\ref{tab:validation-size} asks whether it substitutes for governing the channel on Criteo. It does not, at least over this range. Ours has the smaller signed gap and the higher mean test AUC at every public size. Both advantages are largest at $2{,}000$ examples and shrink as the public set grows, which is what a margin scaled to sampling error should do; at $20{,}000$ examples both methods recover the main-table test AUC, since the training split is fixed across sizes.

\begin{table}[t]
\centering
\small
\setlength{\tabcolsep}{8pt}
\begin{tabular}{@{}llcccc@{}}
\toprule
\textbf{Task} & \textbf{Val. size} & \multicolumn{2}{c}{\textbf{Test} $\uparrow$} & \multicolumn{2}{c}{\textbf{Val.--test gap} $\downarrow$} \\
\cmidrule(lr){3-4}\cmidrule(l){5-6}
 & & \textbf{Exact} & \textbf{Ours} & \textbf{Exact} & \textbf{Ours} \\
\midrule
\multirow{4}{*}{Criteo} & 2{,}000 & 0.7782 $\pm$ 0.0007 & \textbf{0.7816} $\pm$ \textbf{0.0003} & 0.0117 $\pm$ 0.0047 & \textbf{0.0051} $\pm$ 0.0044 \\
 & 5{,}000 & $0.7805 \pm 0.0004$ & \textbf{0.7816} $\pm$ \textbf{0.0003} & 0.0101 $\pm$ 0.0030 & \textbf{0.0081} $\pm$ 0.0032 \\
 & 10{,}000 & $0.7813 \pm 0.0001$ & \textbf{0.7817} $\pm$ \textbf{0.0002} & 0.0086 $\pm$ 0.0009 & \textbf{0.0075} $\pm$ 0.0009 \\
 & 20{,}000 & $0.7818 \pm 0.0002$ & \textbf{0.7821} $\pm$ \textbf{0.0001} & 0.0055 $\pm$ 0.0011 & \textbf{0.0048} $\pm$ 0.0009 \\

\bottomrule
\end{tabular}
\caption{Public-validation size, mean $\pm$ standard error over five seeds. For Criteo, public sets are nested subsets of the original 25k validation split, and the complete 100k training split is fixed across public sizes. Criteo uses the strong hashed logistic-regression configuration from Table~\ref{tab:main-results} and Signed Boosting with at most $30$ matched submissions. Attack strength is selected separately for each seed, public size, and method using validation only.}
\label{tab:validation-size}
\end{table}

\subsection{Fixed margins versus the paired standard error}

\begin{table}[t]
\centering
\small
\setlength{\tabcolsep}{4.5pt}
\begin{tabular}{@{}llccc@{}}
\toprule
\textbf{Task} & \textbf{Update margin}
& \textbf{Stop iter.} $\downarrow$  & \textbf{Test} $\uparrow$
& \textbf{Gap} $\downarrow$ \\
\midrule
\multirow{4}{*}{Criteo CTR}
  & fixed $\eta=0.001$   & $14.6\pm1.9$  & $0.7831\pm0.0013$ & $0.0060\pm0.0006$ \\
  & fixed $\eta=0.005$   & $9.8\pm1.1$   & $0.7684\pm0.0090$ & $0.0066\pm0.0012$ \\
  & fixed $\eta=0.010$   & \textbf{7.8} $\pm$ \textbf{0.4}  & $0.7727\pm0.0025$ & $0.0054\pm0.0003$ \\
  & paired SE (ours)     & $14.2\pm2.4$ &  $\mathbf{0.7840\pm0.0003}$ & $\mathbf{0.0053\pm0.0003}$ \\

\bottomrule
\end{tabular}
\caption{Fixed margin versus the paired standard error. Rows share the harness, prompt, split, budget, and patience, and differ only in the update margin. Proposition~\ref{prop:accuracy} is stated for a fixed $\eta$, whereas the deployed rule computes one.}
\label{tab:margin-ablation}
\end{table}

Table~\ref{tab:margin-ablation} replaces the computed margin with a fixed one on Criteo, asking whether the deployed rule behaves like a well-chosen fixed margin without one having to be chosen. It does: the paired rule achieves both the smallest gap and the highest test AUC on the grid, and it runs as long as any fixed setting, so the advantage is not that it stops sooner.

\subsection{Sensitivity to the agent backbone}

\begin{table}[t]
\centering
\small
\setlength{\tabcolsep}{4pt}
\begin{tabular}{@{}llcccc@{}}
\toprule
\textbf{Agent backbone} & \textbf{Method} & \textbf{Stop iter.} $\downarrow$
& \textbf{Test} $\uparrow$ & \textbf{Val.--test gap} $\downarrow$ \\
\midrule
\multirow{2}{*}{GPT-5.6-Sol (Codex)}
  & Exact & $30.0 \pm 0.0$  & $0.7820 \pm 0.0001$ & $0.0060 \pm 0.0001$ \\
  & Ours  & \textbf{12.3} $\pm$ \textbf{4.4}  & \textbf{0.7823} $\pm$ \textbf{0.0000} & $\mathbf{0.0055 \pm 0.0000}$ \\
\midrule
\multirow{2}{*}{Claude Sonnet 4.6}
  & Exact & $23.4 \pm 3.4$  & $0.7831 \pm 0.0015$ & $0.0058 \pm 0.0001$ \\
  & Ours  & \textbf{22.6} $\pm$ \textbf{1.9}  & $\mathbf{0.7835 \pm 0.0015}$ & $\mathbf{0.0053 \pm 0.0003}$ \\
\bottomrule
\end{tabular}
\caption{Sensitivity to the agent backbone on Criteo CTR, mean $\pm$ standard error over three seeds for Codex; Claude rows retain their separately reported runs. Within each backbone, rows share the same repository, task prompt, split, and outer budget, and differ only in the backbone model and the validation feedback interface. }
\label{tab:agent-sensitivity}
\end{table}

Table~\ref{tab:agent-sensitivity} repeats the Criteo condition on a second backbone. Under Codex the contract cuts the mean stopping round from the cap to twelve while slightly improving test AUC and reducing the gap; those rows match the main-table result and follow the same protocol. Under Claude the two interfaces stop at nearly the same round and both run short of the cap, and the gap narrows by about as much as it does under Codex. Those runs have not been rerun under the current protocol and should be read separately.

\section{Discussion and Limitations}

The no-signal control provides the sharpest evidence. On data where nothing can be learned, an agent given an ordinary prompt and an exact score after every attempt reports an improvement anyway, accumulated slowly enough that no single round looks wrong. This needed no adversary and no unusual agent, only a channel that answers every question. The second finding, which we did not expect, is that exact feedback ran to the budget on every dataset including the one with no signal: an interface that always has something new to report is always worth querying again.
 
The guarantee is narrow, deliberately. It concerns the released transcript rather than the final test score, so the contract makes the reported number trustworthy, not better. It covers a fixed margin, whereas the deployed rule computes one from paired standard errors and is a heuristic. And a thresholded interface returns a selection signal and nothing else, so error analysis has to run on training data or a separately labelled sample. The empirical claim is correspondingly modest: the gap narrows on both real datasets and test performance is slightly higher rather than lower, but the differences are small in absolute terms, three seeds on the real datasets cannot separate a small effect from none, and both datasets are tabular.

\section{Conclusion}

An agent that reads an exact validation score after every attempt is running a selection procedure against a fixed sample, and the score it reports at the end is not what a held-out estimate measures. This is a property of the harness, not of the agent. We proposed a contract that fixes what an agent may see on each split, when its loop should end, and when the test data may be touched, with a release rule that reports a new score only when a submission beats the current best by more than the noise in that comparison. On data with nothing to learn, the contract reports nothing learned while exact feedback reports progress; on real data the gap narrows at no cost; and runs end when improvements run out rather than when the budget does. Harnesses today specify what an agent must produce far more carefully than what it may observe while producing it, and the second question only gets harder as loops grow longer.

\bibliographystyle{plainnat}
\bibliography{references}

\end{document}